\documentclass{article} 

\usepackage[final]{neurips_2021}

\usepackage[utf8]{inputenc} 
\usepackage[T1]{fontenc}    
\usepackage{hyperref}       
\usepackage{url}            
\usepackage{booktabs}       
\usepackage{amsfonts}       
\usepackage{nicefrac}       
\usepackage{microtype}      
\usepackage{xcolor}         

\usepackage{amsmath, amssymb, amsthm}
\usepackage{graphicx}
\usepackage{bbm}

\usepackage{hyperref}

\usepackage{pgf, tikz}
\usetikzlibrary{arrows, automata, positioning}

\newcommand{\R}{\mathbb{R}}
\newcommand{\E}{\mathop{\mathbb{E}}}
\newcommand{\Prob}{\mathbb{P}}
\newcommand{\1}{\mathbbm{1}}

\usepackage{accents}

\usepackage[ruled]{algorithm2e}

\newtheorem{theorem}{Theorem}

\newtheorem{lemma}[theorem]{Lemma}

\usepackage{pifont}
\usepackage{xcolor}

\usepackage{makecell}

\usepackage{thmtools}
\usepackage{thm-restate}

\usepackage{tabularx}

\usepackage{wrapfig}

\usepackage[T1]{fontenc}
\newcommand{\qfil}{\textsc{qfil}}

\usepackage{natbib}
\title{Quantile Filtered Imitation Learning}
\author{%
  David Brandfonbrener \qquad William F. Whitney \qquad Rajesh Ranganath \qquad Joan Bruna\\
  Department of Computer Science, Center for Data Science\\
  New York University \\
  \texttt{david.brandfonbrener@nyu.edu} \\
}

\begin{document}

\maketitle

\begin{abstract}
    We introduce quantile filtered imitation learning (\qfil), a novel policy improvement operator designed for offline reinforcement learning. \qfil\ performs policy improvement by running imitation learning on a filtered version of the offline dataset.
    The filtering process removes $ s,a $ pairs whose estimated Q values fall below a given quantile of the pushforward distribution over values induced by sampling actions from the behavior policy.
    The definitions of both the pushforward Q distribution and resulting value function quantile are key contributions of our method.
    We prove that \qfil\ gives us a safe policy improvement step with function approximation and that the choice of quantile provides a natural hyperparameter to trade off bias and variance of the improvement step.
    Empirically, we perform a synthetic experiment illustrating how \qfil\ effectively makes a bias-variance tradeoff and we see that \qfil\ performs well on the D4RL benchmark. 
\end{abstract}

\section{Introduction}

Offline RL offers tantalizing promise in diverse applications from robotics to healthcare \citep{levine2020offline}.
However, offline RL has some fundamental limitations. In particular, policies cannot in general extrapolate beyond the coverage of the dataset since unknown parts of state-action space may be dangerous. 
Subject to this safety constraint, algorithms still face a bias-variance tradeoff when attempting to perform a policy improvement step. 
Explicitly, we can reduce the variance of our learned policy by remaining closer to the behavior policy where we have more data. But this comes at the cost of bias away from the optimal policy and towards the behavior.

In this paper we propose a novel policy improvement operator called quantile filtered imitation learning (\qfil) to safely make this bias-variance tradeoff. Our improvement operator can be coupled with any value estimation technique to get an offline RL algorithm. 
Simply put, \qfil\ attempts to imitate actions from the dataset that perform well while ignoring those that perform poorly. To decide which actions to imitate we use the pushforward distribution induced by pushing samples from the behavior policy through the estimated Q function and then only imitate actions whose Q estimates exceed some quantile of the pushforward distribution. 
Importantly, since the policy learning step is simply imitation learning, we ensure that the policy has no incentive to choose actions outside of the data distribution, which provides safety. 
Selecting a high quantile allows us to make a less biased and more aggressive update by imitating a smaller subset of the data, while a low quantile provides a lower-variance update.

We provide both theoretical and empirical arguments for the efficacy of \qfil. On the theory side, we prove a safe policy improvement guarantee that illustrates how the quantile $ \tau$ controls the bias-variance tradeoff. On the empirical side, we first provide a toy experiment that demonstrates how the quantile $ \tau$ controls the bias-variance tradeoff. Then we demonstrate that \qfil\ can achieve performance competitive with the state of the art on the D4RL \citep{fu2020d4rl} benchmark on the MuJoCo tasks using one-step on-policy value estimation and on the Ant Maze tasks using iterative off-policy value estimation.


\begin{figure}[t!]
    \centering
    \includegraphics[width=0.9\textwidth]{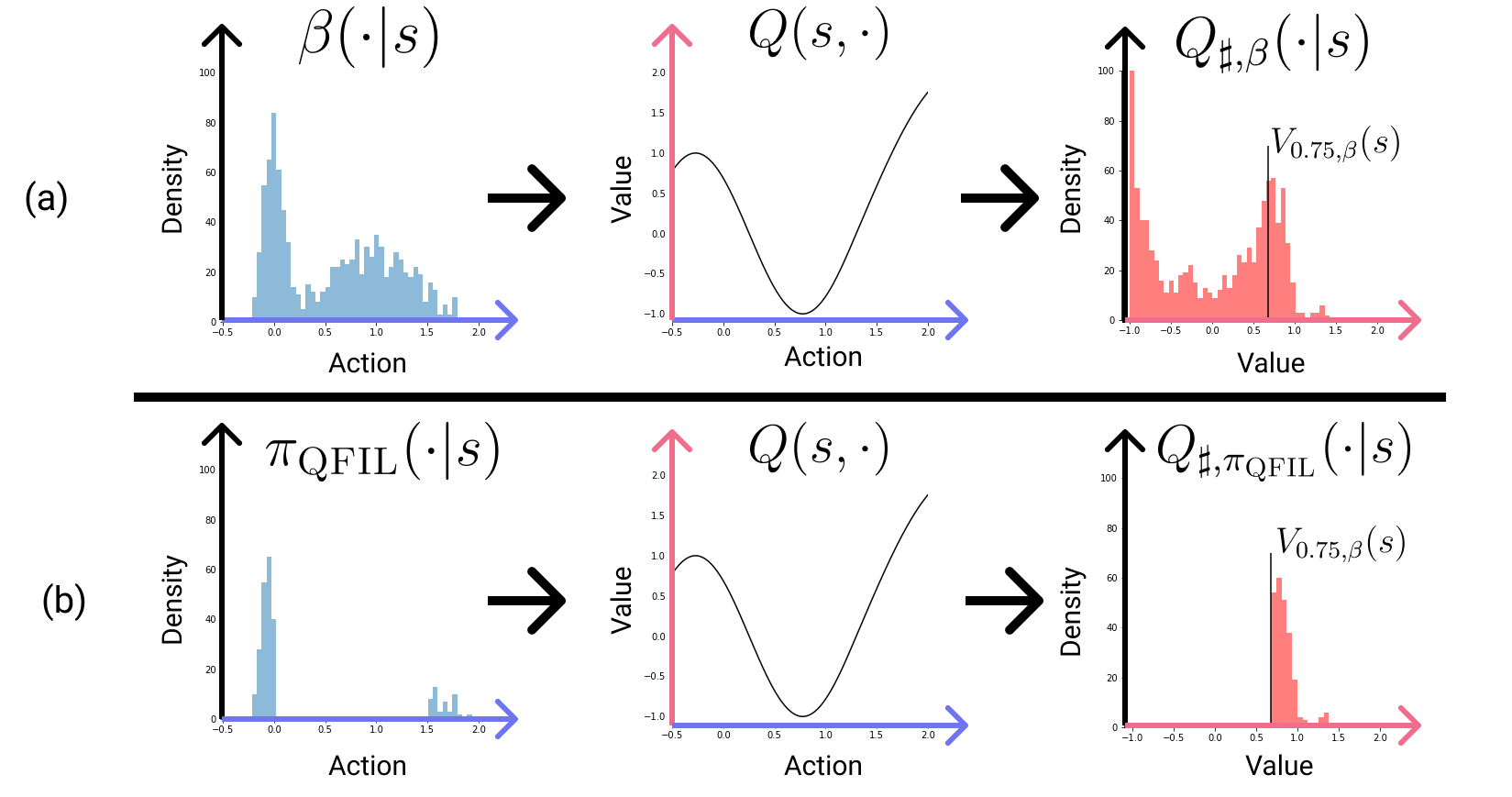}
    \caption{(a) An illustration of the pushforward distribution $ Q_{\sharp, \beta}(\cdot|s)$ induced by pushing samples from $ \beta(\cdot | s)$ through $Q(s,\cdot)$. $ Q_{\sharp, \beta}(\cdot|s)$ is the distribution of predicted Q-values generated by sampling an action from the policy $\beta$ and evaluating it using $Q$. The value function quantile $ V_{0.75, \beta}(s)$ is the 0.75 quantile of this pushforward distribution. (b) An illustration of how \qfil\ defines an improved policy $ \pi_{\qfil}$ that only selects actions for which $ Q(s,a) \geq V_{0.75, \beta}(s)$.}
    \label{fig:quantile}
\end{figure}

\section{Setup}\label{sec:setup}

\begin{wrapfigure}[9]{r}{0.6\textwidth}
\vspace{-0.6cm}
\begin{algorithm}[H]
    \SetKwInOut{Input}{input}
    \Input{$ K$, dataset $ D_N$,  estimated behavior $ \hat \beta$}
             
    Set $\pi_0 = \hat \beta$. Initialize $ \widehat Q^{\pi_{-1}}$ randomly. 
             
    \For{k = 1, \dots, K}{
        Policy evaluation: $ \widehat Q^{\pi_{k-1}} = \mathcal{E}(\pi_{k-1}, D_N, \widehat Q^{\pi_{k-2}})$ 
              
        Policy improvement: $ \pi_{k} = \mathcal{I}(\widehat Q^{\pi_{k-1}}, \hat \beta, D_N, \pi_{k-1})$
        }
    \caption{OAMPI}
    \label{alg:oampi}
\end{algorithm}
\end{wrapfigure}

We will consider an offline RL setup as follows. Let $ \mathcal{M} = \{\mathcal{S}, \mathcal{A}, P, R, \gamma \}$ be a discounted infinite-horizon  MDP. Define $ r(s,a) = \E[R(s,a)]$.
Rather than interacting with $\mathcal{M}$, we only have access to a dataset $ \mathcal{D}$ of $ N $ tuples of $ (s_i, a_i, r_i)$ collected by trajectories from some behavior policy $ \beta$ with initial state distribution $ \rho$. 
Define $ d_\pi(s) := (1-\gamma) \sum_{t=0}^\infty \gamma^t \Prob_{\rho, P, \pi}(s_t = s)$ to be the discounted state visitation distribution under a policy $ \pi$. The objective is to maximize the expected return $ J $ of the learned policy $J(\pi) := \E_{\rho, P, \pi}\left[ \sum_{t=0}^\infty \gamma^t r(s_t, a_t) \right]$.

Following \cite{brandfonbrener2021offline} we consider a generic template for model free offline RL algorithms as offline approximate modified policy iteration (OAMPI). This template, shown in Algorithm \ref{alg:oampi}, alternates between steps of policy evaluation and policy improvement using the fixed offline dataset. This paper focuses on proposing a specific version of the improvement step that can be used in tandem with various evaluation steps. Moreover, as we will show, this improvement operator can be useful both in the one-step ($ K=1$) and iterative ($K >> 1$) regimes. This template is generic enough to capture essentially all related work by just replacing the $ \mathcal{E}$ and $ \mathcal{I}$ operators as we will discuss in more detail in Section \ref{sec:related} after introducing \qfil.

\section{Our algorithm: Quantile Filtered Imitation Learning (\qfil)}

In this section we will introduce our main algorithmic contributions. First, we will formally define value function quantiles in terms of pushforward distributions over Q values. Then we define quantile filtered imitation learning (\qfil). Finally, we provide some theoretical analysis of \qfil\ to demonstrate how the quantile governs a bias-variance tradeoff in the policy improvement step. 

\subsection{Pushforward Q distributions and value function quantiles}

The main mathematical objects that we introduce in this work are the pushforward Q distribution $ Q_{\sharp, \pi}(\cdot|s)$ for some action-value function $ Q$ and policy $ \pi$ and the resulting value function quantile $ V_{\tau, \pi}(s)$ at quantile $ \tau$. These concepts are illustrated in Figure \ref{fig:quantile}. Intuitively, this pushforward distribution $ Q_{\sharp, \pi}(\cdot |s)$ is the distribution over $ \R$ that is generated when we sample actions from $ \pi$ conditioned on $ s $ and then ``push'' those actions ``forward'' through the function $ Q(s, \cdot)$. Then $ V_{\tau, \pi}$ simply takes a quantile of this induced distribution over $ \R$. 

More formally, we can define the pushforward distribution over values at a state $ s $ induced by any $ \pi $ and $ Q $ evaluated at any measurable set of values $ B \subseteq \R$ as:
\begin{align}
    Q_{\sharp, \pi}(B|s) := \pi(Q^{-1}(B;s)| s)
\end{align}
where $ Q^{-1}$ gives the pre-image of $ B $ under $ Q$ so that for any $ v \in B$, we have $ Q^{-1}(v;s) \subseteq \mathcal{A}$ and for any $ a \in Q^{-1}(v;s)$, we have $ Q(s,a) = v$.

Equivalenty, since $Q_{\sharp, \pi}(\cdot|s)$ is a real-valued distribution, we can define it in terms of its CDF. Let $ X \sim Q_{\sharp, \pi}(\cdot|s)$, then we have that 
\begin{align}
    \Prob(X \leq v)  = \Prob_{ a \sim \pi|s}(Q(s,a) \leq v).
\end{align}
Now we can directly define the value function quantile $ V_{\tau, \pi}(s)$ by taking the $ \tau $ quantile of the pushforward Q distribution:
\begin{align}
    V_{\tau, \pi}(s) :&= \sup \left \{ v\in \R \ \ s.t.\ \ \Prob_{X \sim Q_{\sharp, \pi}(\cdot|s)}(X \leq v) \leq \tau \right\} \\&= \sup \left \{ v\in \R \ \ s.t.\ \ \Prob_{a \sim \pi|s}(Q(s,a) \leq v) \leq \tau \right\} 
\end{align}

\paragraph{Comparison to distributional RL.} We want to emphasize that the value function quantile is a very different object from the quantile value function considered in work on distributional RL \citep{dabney2018distributional}. That work attempts to approximate the full distribution of stochastic returns $ Z^\pi$ for some policy $ \pi$ by computing various quantiles of the distribution over $ Z^\pi$ conditioned on $ s,a$ using a bellman backup. These objects consider stochasticity over \emph{entire trajectories}. In contrast, we are using the standard Q functions and the pushforward distribution over Q values $ Q_{\sharp, \pi}(\cdot|s)$ at $ s$ only considers stochasticity induced by $ \pi$ at the first step in a trajectory.
 
\subsection{The \qfil\ policy improvement operator}

\qfil\ defines a novel policy improvement operator that can be incorporated into the generic OAMPI framework introduced in Section \ref{sec:setup}. 
The operator simply filters out data where the estimated Q value falls below the $ \tau$ value function quantile. Explicitly, we define the policy update as follows:
\begin{align}
    \mathcal{D}_\tau = \left\{s, a \in \mathcal{D} \ \ s.t.\ \ \widehat Q^k(s,a) > \widehat V^{k}_{\tau, \beta}(s) \right\}, \qquad \pi_{k+1} = \arg\max_{\pi} \sum_{s,a \in \mathcal{D_\tau}} \log \pi(a|s). 
\end{align}

Note that importantly, we use the pushforward distribution induced by $ \beta$ to compute the value function quantile. This ensures that we are evaluating our estimated value functions on data from the data-generating distribution.
Also note that we can combine the \qfil\ policy improvement operator with any policy evaluation operator to get $ \widehat Q^k$. 

\paragraph{Practical benefits of \qfil.} There are three main benefits to the \qfil\ improvement operator. 
\begin{enumerate}
    \item Safety: the \qfil\ objective only encourages imitation of actions already in the dataset. This is useful in domains where we know that the behavior policy is safe but that there may be unsafe actions if we attempt to extrapolate beyond what we have seen in the data.
    \item Bias-variance tradeoff: the quantile provides an effective and intuitive knob to trade off bias and variance. As we will see in theory and practice below, the quantile essentially allows us to choose how much data we use to estimate a new policy. We can get a small sample of data from a near optimal policy, but at the cost of high variance. Or we can use a large sample, but at the cost of a bias towards the behavior policy and away from an optimal policy. The quantile is intuitive since for example setting $ \tau = 0.9$ means we expect to have about 10\% of the data from $ \mathcal{D}$ present in $ \mathcal{D}_\tau$. A practitioner can usually have a decent sense of how much data is needed to reliably solve the supervised policy estimation problem and can use this to propose reasonable settings of the quantile. 
    \item Optimization: the \qfil\ policy learning step only uses standard supervised learning. This means we get access to all the tools, tricks, and hyperparameters that have been designed for optimizing supervised learning problems. In contrast, policy improvement operators that attempt to optimize Q directly 
    face a different optimization problem that machine learning tools have not been optimized for. 
\end{enumerate}

\paragraph{Implementation-level decisions.} The main implementation-level decisions left to instantiate the algorithm are to (1) decide how to estimate $ Q $, and (2) how to estimate $ V_{\tau, \beta}$. For $Q$ we experiment with both on-policy estimates of $ Q^\beta$ using SARSA Q estimation and off-policy estimates of $ Q^{\pi_k}$ using uncorrected DDPG-style Q estimation where we sample $ a' \sim \pi_k$. To compute $ V_{\tau, \beta}$ we use a sampling-based approach. Specifically, we take $ M$ samples from an estimated behavior policy $ \hat \beta$ and then compute the empirical quantile after pushing the $ M $ samples through $ Q$. One final note is that when implementing \qfil\ we can replace the explicit computation of $ \mathcal{D}_\tau$ by computing 0-1 weights for each datapoint as $ \1[\widehat Q^k(s,a) > \widehat V^{k}_{\tau, \beta}(s)]$.  

\subsection{Theoretical analysis}

Now that we have laid out the algorithm we will provide some brief theoretical justification for \qfil. Specifically, we will show that the quantile provides an effective way to trade off the variance of the improvement step caused by finite data and errors in the approximation and estimation of the Q function and policy with bias induced by imitating suboptimal actions. 
Here we will only prove a guarantee for the one-step variant of the algorithm that uses on-policy value estimation of $ Q^\beta$, and we leave a treatment of off-policy learning for future work. 

To make our analysis we will need a few definitions and assumptions. Namely,
\begin{enumerate}
    \item Assume value estimation error is bounded $ \E_{\substack{s\sim d_\beta\\ a\sim \beta|s}} [(Q^\beta(s,a) - \widehat Q^\beta(s,a))^2] \leq \varepsilon_Q(N)$, whp.
    \item Assume that for any policy $ \pi$ we can produce an estimate $ \hat \pi$ based on $ N $ samples of $ s \sim d, a \sim \pi. | s$ such that: $\E_{s\sim d}[D_{TV}(\hat \pi(\cdot|s)\| \pi(\cdot|s))] \leq \varepsilon_\pi(N)$, whp.
    \item Assume value function quantile estimation such that $ \Prob_{a \sim \beta|s}(\widehat Q^\beta(s,a) \leq \widehat V_{\tau,\beta}(s)) = \tau$.\footnote{This assumption is stronger than the first two. We conjecture that it can be weakened to gracefully allow for approximation of the value function quantile, but leave this to future work.}
    \item Let $ Q^\beta_{\max} = \sup_{s,a} Q^\beta(s,a)$ and $ A^{\beta}_{\max} = \sup_{s,a} A^\beta(s,a)$
    \item Define the \qfil\ policy for $ \widehat Q^\beta$ by $\pi_\tau(a|s) := \beta(a|s) \frac{\1[\widehat Q^\beta(s,a) \geq \widehat V_{\tau, \beta}(s)]}{1 - \tau}$.
\end{enumerate}

Now we are ready to state the main result which is a lower bound on the expected improvement of our learned policy $ \hat \pi$ over the behavior $ \beta$. The proof is deferred to Appendix \ref{sec:proof}.

\begin{restatable}{proposition}{advantage}\label{thm:advantage}
Under the above assumptions and letting $ W_1$ denote the Wasserstein-1 distance, \qfil\ returns a policy $ \hat \pi$ such that whp
\begin{align}
   \left(J(\hat \pi) - J(\beta)\right) &(1-\gamma) \geq \E_{s\sim d_\beta} \left[ W_1\left(\widehat Q^\beta_{\sharp, \pi_\tau}(\cdot|s), \ \  \widehat Q^\beta_{\sharp, \beta}(\cdot|s)\right) \right]\\&- 2 \left(\sqrt{\frac{ \varepsilon_Q(N)}{1-\tau}} \ +\  \frac{\varepsilon_\pi((1-\tau) N)Q^\beta_{\max} }{1-\tau} + \frac{ \gamma A^\beta_{\max}}{1-\gamma}\bigg(\tau + \varepsilon_\pi((1-\tau) N)\bigg)\right). 
\end{align}
\end{restatable}

Let's examine each term in the bound. The first Wasserstein distance term expresses the inverse of the bias: it is large for large values of $ \tau$ that approach higher value policies and it decreases to 0 for lower values of $ \tau$ that approach pure imitation. In other words, this term quantifies the improvement of the \qfil\ policy $ \pi_\tau$ over the behavior $ \beta$ according to the estimated Q function $ \widehat Q^\beta$. Note that because the Wasserstein distance is always at least zero, so is this term. Perhaps this term is easiest to understand in pictures as it measures the distance between the real-valued distributions shown in red in the rightmost panels of Figure \ref{fig:quantile}. Setting a higher quantile $ \tau$ will move more mass to higher estimated values from the pushforward behavior distribution $ \widehat Q^\beta_{\sharp, \beta}(\cdot|s)$ to get to the pushforward \qfil\ distribution $ \widehat Q^\beta_{\sharp, \pi_\tau}(\cdot|s)$. The precise scaling with $ \tau$ will be highly dependent on $ \widehat Q^\beta$ and $ \beta$. 

The second term encapsulates the variance, which is the error due to the finite sample size, approximation error of the Q function, approximation error of our imitation learning, and effects of distribution shift.
Importantly, this term becomes more negative for larger values of $ \tau$. Explicitly, if we assume a $ \frac{1}{\sqrt{N}}$ scaling for the $ \varepsilon$ term then we can expect the variance bound to scale with $ \frac{1}{(1-\tau)^{3/2}} + \tau$.  This captures the intuition that staying closer to the behavior by setting a small value of $ \tau$ will reduce the variance. This bound is likely often much too pessimistic about the impacts of distribution shift, but it is beyond the scope of this paper to provide a more intricate problem-dependent analysis.

\section{Related work}\label{sec:related}

The most closely related line of work presents various filtered and weighted imitation learning improvement operators. This includes MARWIL \citep{wang2018exponentially}, CRR \citep{wang2020critic}, AWR \citep{peng2019advantage}, AWAC \citep{nair2020accelerating}, BAIL \citep{chen2020bail}, and ABM \citep{Siegel2020Keep}.
While all of these algorithms have various differences, most of them perform policy improvement using exponentially weighted imitation learning similar to:
\begin{align}
    \pi_{k+1} = \arg\max_{\pi} \sum_{s,a \in \mathcal{D}} \exp[\alpha (Q(s,a) - V(s))] \log \pi(a|s).
\end{align}
The key difference between \qfil\ and these papers is the introduction of the value function quantile into the advantage calculation. The quantile $ \tau$ replaces the hyperparameter $ \alpha$ from the exponentially weighted family of algorithms and provides a more graceful way to trade off bias and variance. 
In some sense this is a numerical issue where larger values of $ \alpha$ can cause instability due to very large weights. In practice implementations need to clip the weights and then as $ \alpha$ increases the weight just approaches the hard threshold $ \1[Q(s,a) - V(s) > 0]$ multiplied by some constant. 
Moreover, the approaches are not mutually exclusive and can be combined by using the value function quantile $ V_{\tau, \beta} $ in place of  $V $, but at the cost of now having to tune bothe $ \tau$ and $ \alpha$.

Another relevant piece of related work is \cite{chen2021decision} which in addition to proposing the decision transformer proposes the \%BC baseline. This baseline runs filtered imitation learning where the epsiodic return is used in place of the Q function and a constant value function that does not depend on the state, but does try to estimate a quantile of the distribution over returns. In contrast, \qfil\ uses a learned Q function and state-dependent value function quantile estimate. 

The other relevant detail is the use of hard filtering instead of expoential weighting. This is motivated by \cite{byrd2018effect} which shows that importance weighting is ineffective when training large, flexible neural network models. Some algorithms (CRR, ABM, and BAIL, \%BC) also use hard filtering like \qfil. None of them use value function quantiles. CRR proposes a variant called CRR-max that uses the max of $ m$ samples to effectively estimate the $ \frac{m-1}{m}$ quantile. BAIL uses the ``upper envelope'' of the data, which relies more on generalization of function approximation and as a result may exclude all datapoints at some states when the estimate of the upper envelope exceeds the 1.0 quantile of the pushforward distribution. \qfil\ provides a more flexible and consistent mechanism to trade off bias and variance. An extended discussion of other, less closely related offline policy improvement operators can be found in Appendix \ref{sec:related_cont}.

In concurrent work, \cite{kostrikov2021offlineb} propose a similar method that uses exponentiated expectile advantage functions to weight the imitation loss. However, instead of first learning a standard Q function and then computing a value function quantile, they modify the Bellman backup directly so that the Q estimates are not estimating the Q values of a particular policy.
Moreover, at an implementation level, we use hard filtering and estimate quantiles from samples, while \cite{kostrikov2021offlineb} use exponential weighting and expectile regression. Finally, because our algorithm does not modify the Bellman backup, it is simpler to provide a more rigorous theoretical analysis.


\section{Experiments}

\subsection{Synthetic experiment}

To illustrate the bias-variance tradeoff we discussed in the prior section, we created a simple synthetic problem. Since we are focused on the policy improvement step rather than the evaluation step, we will use a finite horizon environment with horizon equal to 1 (i.e. a contextual bandit). We use real-valued, continuous, 1-dimensional states and actions and learn $ \hat \beta, \widehat Q^,$ and $ \hat \pi$ with small neural networks of width 50 and depth 2. The data is generated according to:
\begin{align*}
    s \sim U([0,1]), \qquad a|s \sim \frac{s + \epsilon}{2}, \epsilon \sim U([0,1]), \qquad r(s,a) = \begin{cases}1 - |a - (1-s)|& a \in [\frac{s}{2}, \frac{s+1}{2}]\\ -1 & otherwise  \end{cases}
\end{align*}
Note that the reward function penalizes actions with zero probability under the behavior. 

\begin{figure}[h]
    \centering
    \includegraphics[width=0.9\textwidth]{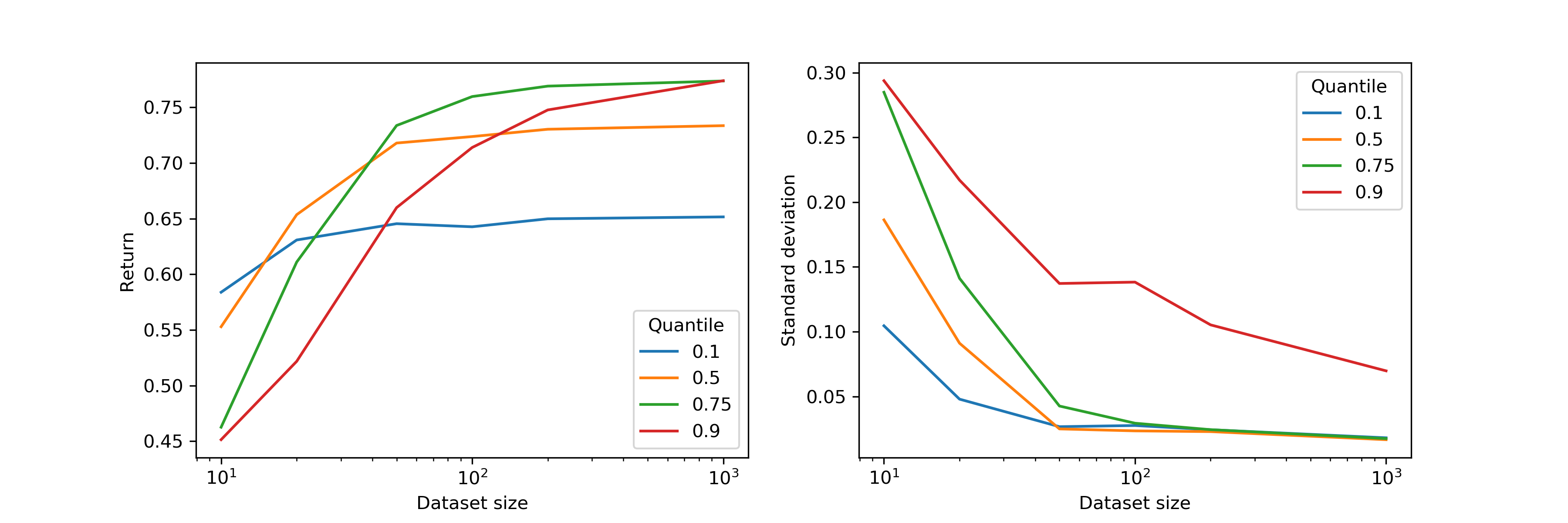}
    \caption{Figures showing the mean (left) and standard deviation (right) across 50 random seeds of the return of the policies learned by \qfil\ on the synthetic problem with various quantiles and datasets. The plots illustrate how the quantile $ \tau$ governs a bias-variance tradeoff.}
    \label{fig:toy}
\end{figure}

We simulate several experiments with various dataset sizes and quantiles. Full experimental details can be found in Appendix \ref{sec:details}. We visualize the results in Figure \ref{fig:toy} demonstrating that higher quantiles have more variance but less bias. As a result, different quantiles are optimal depending upon the dataset size.

\subsection{D4RL experiments}

Now we move on to larger scale experiments on the D4RL benchmark \citep{fu2020d4rl}. Full experimental details can be found in Appendix \ref{sec:details}.

\paragraph{One-step MuJoCo.} Recent work of \cite{brandfonbrener2021offline} has showed that simply performing one step of policy improvement ($ K=1$ in OAMPI) yields state of the art performance on most of the MuJoCo tasks in the D4RL benchmark excepting the ``random'' datasets where iterative algorithms perform better. Our first experiment shows that replacing the exponentiated advantage function used by MARWIL and CRR yields slightly better performance on many of the MuJoCo tasks and can be seen in Table \ref{tab:onestep-mujoco}. We also compare to a \% BC baseline that just runs behavior cloning on the trajectories with the highest returns (where we sweep over different values for the amount of data to keep). By using only the one-step algorithm this experiment attempts to isolate the effect of the \qfil\ operator for policy improvement from effects of off-policy evaluation.  

    \begin{table}[h]
    \centering
    \caption{Results for one-step algorithms on MuJoCo tasks from the D4RL benchmark. We run 3 seeds and tune over 4 hyperparameter values. We report the mean and standard deviation across seeds using 100 evaluation episodes for the best hyperparameter. Exp-adv refers to exponentiated advantage as in MARWIL or CRR, here limited to one step.}
    \begin{small}
    \begin{tabular}{lcccc}
        \toprule
        Dataset & BC & \% BC & One-step Exp-adv & One-step \qfil \\
        \midrule
        halfcheetah-med & 41.9  $\pm$  0.1 & 42.3  $\pm$  0.2
                        & 48.4  $\pm$  0.1
                        & \textbf{50.6  $\pm$  0.0} \\
        walker2d-med & 68.6  $\pm$  6.3 & 73.7  $\pm$  1.5
                        & 81.8  $\pm$  2.2
                        & \textbf{83.7  $\pm$  0.9} \\
        hopper-med & 49.9  $\pm$  3.1 & 57.5  $\pm$  1.5 
                       & 59.6  $\pm$  2.5
                       & \textbf{62.8  $\pm$  2.0} \\
        \midrule
        halfcheetah-med-rep & 34.3  $\pm$  0.7 & 37.8  $\pm$  0.4 
                        & 38.1  $\pm$  1.3 
                        & \textbf{40.4  $\pm$  0.8} \\
        walker2d-med-rep &  26.2  $\pm$  2.4 & \textbf{61.8  $\pm$  1.8}
                        & 49.5  $\pm$  12.0 
                        & 55.0  $\pm$  0.8 \\
        hopper-med-rep & 23.2  $\pm$  2.8  & 77.0  $\pm$  5.2
                        & \textbf{97.5  $\pm$  0.7} 
                        & 94.6  $\pm$  2.3 \\
        \midrule
        halfcheetah-med-exp & 61.1  $\pm$  2.7 & 92.0  $\pm$  0.3
                        & 93.4  $\pm$  1.6
                        & \textbf{94.4  $\pm$  0.5} \\
        walker2d-med-exp & 78.5  $\pm$  22.4  & 108.8  $\pm$  0.4
                        & \textbf{113.0  $\pm$  0.4}
                        & 112.7  $\pm$  0.7\\
        hopper-med-exp & 49.1  $\pm$  4.3 & \textbf{110.1  $\pm$  0.6}
                        & 103.3  $\pm$  9.1
                        & 107.6  $\pm$  2.2\\
        \midrule          
        halfcheetah-rand &  2.2  $\pm$  0.0 & 2.2  $\pm$  0.1 
                        & 3.2  $\pm$  0.1
                        & \textbf{5.6  $\pm$  0.3}\\
        walker2d-rand & 0.9  $\pm$  0.1 & 2.4  $\pm$  0.1 
                        & 5.6  $\pm$  0.8
                        & \textbf{6.0  $\pm$  0.5}\\
        hopper-rand & 2.0  $\pm$  0.1 & 5.1  $\pm$  1.9
                        &  \textbf{7.5  $\pm$  0.4}
                        & 7.0  $\pm$  0.5\\
        \bottomrule
    \end{tabular}
    \end{small}
    \label{tab:onestep-mujoco}
    \end{table}

\begin{table}[h]
    \centering
    \caption{Results for one-step algorithms on antmaze tasks from the D4RL benchmark. Again, we run 3 seeds, tune over 4 hyperparameter values, and use 100 eval episodes.}
    \begin{small}
    \begin{tabular}{lcccc}
        \toprule
        Dataset & BC & \% BC & One-step Exp-Adv & One-step \qfil \\
        \midrule
        umaze & 60.3  $\pm$  7.8 & 57.3  $\pm$  2.1 
                        & 75.7  $\pm$  5.6
                        & \textbf{91.0  $\pm$  1.4 } \\
        umaze-diverse & 54.0  $\pm$  3.7 & 65.3  $\pm$  6.3
                        & \textbf{66.0  $\pm$  4.9}
                        & 61.7  $\pm$  9.5 \\
        \midrule
        medium-play & 0.3  $\pm$  0.5 & 0.7  $\pm$  0.9 
                        & 0.7  $\pm$  0.9
                        & \textbf{13.3  $\pm$  8.4} \\
        medium-diverse & 0.0  $\pm$  0.0 & 0.7  $\pm$  0.5
                        & 1.0  $\pm$  1.4
                        & \textbf{4.0  $\pm$  4.2} \\
        \midrule          
        large-play & 0.0  $\pm$  0.0 & 0.0  $\pm$  0.0
                        & 0.0  $\pm$  0.0
                        & \textbf{1.3  $\pm$  1.9} \\
        large-diverse & 0.0  $\pm$  0.0 & 0.3  $\pm$  0.5 
                        & 0.0  $\pm$  0.0
                        & \textbf{0.7  $\pm$  0.9} \\
        \bottomrule
    \end{tabular}
    \end{small}
    \label{tab:onestep-antmaze}
\end{table}

\paragraph{One-step antmaze.} While the Mujoco tasks see the best performance with one-step methods, this can be attributed to the dense rewards and low-coverage behavior policies. In contrast, the antmaze tasks in the D4RL benchmark have sparse rewards and diverse behavior policies, giving a different sort of challenge to offline RL algorithms.
Results are shown in Table \ref{tab:onestep-antmaze}. We see larger gains from \qfil\ over the baselines on the harder tasks.

\begin{table}[h]
    \centering
    \caption{Results for iterative algorithms on antmaze tasks from the D4RL benchmark. This time we run 3 seeds, tune over 4 hyperparameter values, and use 100 eval episodes.}
    \begin{small}
    \begin{tabular}{lcccc}
        \toprule
        Dataset &  Iterative Exp-Adv & Iterative \qfil & CQL & IQL\\
        \midrule
        umaze 
                        & 61.3  $\pm$  6.2
                        & \textbf{96.7  $\pm$  1.2}  
                        & 74.0 
                        & 87.5 \\
        umaze-diverse 
                        & 70.0  $\pm$  3.7
                        & 67.7  $\pm$  0.9 
                        & \textbf{84.0}
                        & 62.2 \\
        \midrule
        medium-play 
                        & 0.3  $\pm$  0.5 
                        & 60.0  $\pm$  8.0 
                        & 61.2
                        & \textbf{71.2} \\
        medium-diverse 
                        & 0.0  $\pm$  0.0 
                        & 61.3  $\pm$  19.0
                        & 53.7
                        & \textbf{71.0} \\
        \midrule          
        large-play 
                        & 0.3  $\pm$  0.5 
                        & 4.0  $\pm$  5.7 
                        & 15.8
                        & \textbf{39.6} \\
        large-diverse 
                        & 0.3  $\pm$  0.5 
                        & 12.7  $\pm$  17.9 
                        & 14.9
                        & \textbf{47.5} \\
        \bottomrule
    \end{tabular}
    \end{small}
    \label{tab:iter-antmaze}
\end{table}

\paragraph{Iterative antmaze} Now we go beyond one-step updates and show that \qfil\ can be effectively combined with off-policy Q estimation. To do this we simply instantiate the OAMPI algorithm using \qfil\ as the improvement operator and using standard SARSA-style fitted Q evaluation with target networks as the evaluation operator. Importantly, since the \qfil\ policies are derived from imitation, they remain within the data distribution to allow for easier off-policy evaluation. Moreover, since we compute value function quantiles by pushing forward the estimated behavior, those quantiles are also not being queried out of distribution. Results in Table \ref{tab:iter-antmaze} show that \qfil\ substanitally outperforms the exponentiated advantage algorithm and achieves performance competitive with state of the art algorithms CQL \citep{kumar2020conservative} and IQL \citep{kostrikov2021offlineb} that modify the Bellman backups (where results for those algorithms are taken directly from those papers). Note that this iterative algorithm is not strictly covered by our theory from the previous section which only considers the one-step approach. It is an interesting direction for future work to get better theoretical guarantees for iterative algorithms.

\section{Discussion}

Here we have presented \qfil, a novel policy improvement operator for offline RL. \qfil\ provides a safe policy improvement step that only imitates actions already in the dataset and does so in a numerically stable way that only uses supervised learning while still providing a useful knob in the quantile to trade off bias and variance. 
Along the way we offered definitions of the pushforward Q distribution and value function quantile that we hope can find uses beyond \qfil. Indeed, anywhere that the advantage function is used, it could be replaced by a quantile advantage that uses a value function quantile instead of the standard value function.

\subsection*{Acknowledgements}
We would like to thank Ilya Kostrikov for suggestions about the antmaze datasets and the anonymous reviewers for their suggestions about related work.
This work is partially supported by the Alfred P. Sloan Foundation, NSF RI-1816753, NSF CAREER CIF 1845360, NSF CHS-1901091, Samsung Electronics, and the Institute for Advanced Study.
DB is supported by the Department of Defense (DoD) through the National Defense Science \& Engineering Graduate Fellowship (NDSEG) Program.


\newpage

\bibliography{rl.bib}

\begin{thebibliography}{25}
\providecommand{\natexlab}[1]{#1}
\providecommand{\url}[1]{\texttt{#1}}
\expandafter\ifx\csname urlstyle\endcsname\relax
  \providecommand{\doi}[1]{doi: #1}\else
  \providecommand{\doi}{doi: \begingroup \urlstyle{rm}\Url}\fi

\bibitem[Achiam et~al.(2017)Achiam, Held, Tamar, and
  Abbeel]{achiam2017constrained}
Joshua Achiam, David Held, Aviv Tamar, and Pieter Abbeel.
\newblock Constrained policy optimization.
\newblock In \emph{International Conference on Machine Learning}, pages 22--31.
  PMLR, 2017.

\bibitem[Brandfonbrener et~al.(2021)Brandfonbrener, Whitney, Ranganath, and
  Bruna]{brandfonbrener2021offline}
David Brandfonbrener, William~F Whitney, Rajesh Ranganath, and Joan Bruna.
\newblock Offline rl without off-policy evaluation.
\newblock In \emph{Advances in Neural Information Processing Systems}, 2021.

\bibitem[Byrd and Lipton(2018)]{byrd2018effect}
Jonathon Byrd and Zachary~C Lipton.
\newblock What is the effect of importance weighting in deep learning?
\newblock \emph{arXiv preprint arXiv:1812.03372}, 2018.

\bibitem[Chen et~al.(2021)Chen, Lu, Rajeswaran, Lee, Grover, Laskin, Abbeel,
  Srinivas, and Mordatch]{chen2021decision}
Lili Chen, Kevin Lu, Aravind Rajeswaran, Kimin Lee, Aditya Grover, Michael
  Laskin, Pieter Abbeel, Aravind Srinivas, and Igor Mordatch.
\newblock Decision transformer: Reinforcement learning via sequence modeling.
\newblock \emph{arXiv preprint arXiv:2106.01345}, 2021.

\bibitem[Chen et~al.(2020)Chen, Zhou, Wang, Wang, Wu, and Ross]{chen2020bail}
Xinyue Chen, Zijian Zhou, Zheng Wang, Che Wang, Yanqiu Wu, and Keith Ross.
\newblock Bail: Best-action imitation learning for batch deep reinforcement
  learning.
\newblock \emph{Advances in Neural Information Processing Systems}, 33, 2020.

\bibitem[Dabney et~al.(2018)Dabney, Rowland, Bellemare, and
  Munos]{dabney2018distributional}
Will Dabney, Mark Rowland, Marc Bellemare, and R{\'e}mi Munos.
\newblock Distributional reinforcement learning with quantile regression.
\newblock In \emph{Proceedings of the AAAI Conference on Artificial
  Intelligence}, volume~32, 2018.

\bibitem[Fu et~al.(2020)Fu, Kumar, Nachum, Tucker, and Levine]{fu2020d4rl}
Justin Fu, Aviral Kumar, Ofir Nachum, George Tucker, and Sergey Levine.
\newblock D4rl: Datasets for deep data-driven reinforcement learning.
\newblock \emph{arXiv preprint arXiv:2004.07219}, 2020.

\bibitem[Fujimoto and Gu(2021)]{fujimoto2021minimalist}
Scott Fujimoto and Shixiang~Shane Gu.
\newblock A minimalist approach to offline reinforcement learning.
\newblock \emph{arXiv preprint arXiv:2106.06860}, 2021.

\bibitem[Fujimoto et~al.(2018)Fujimoto, Meger, and Precup]{fujimoto2018off}
Scott Fujimoto, David Meger, and Doina Precup.
\newblock Off-policy deep reinforcement learning without exploration.
\newblock \emph{arXiv preprint arXiv:1812.02900}, 2018.

\bibitem[Kingma and Ba(2014)]{kingma2014adam}
Diederik~P Kingma and Jimmy Ba.
\newblock Adam: A method for stochastic optimization.
\newblock \emph{arXiv preprint arXiv:1412.6980}, 2014.

\bibitem[Kostrikov et~al.(2021{\natexlab{a}})Kostrikov, Nair, and
  Levine]{kostrikov2021offlineb}
Ilya Kostrikov, Ashvin Nair, and Sergey Levine.
\newblock Offline reinforcement learning with implicit q-learning.
\newblock \emph{arXiv preprint arXiv:2110.06169}, 2021{\natexlab{a}}.

\bibitem[Kostrikov et~al.(2021{\natexlab{b}})Kostrikov, Tompson, Fergus, and
  Nachum]{kostrikov2021offline}
Ilya Kostrikov, Jonathan Tompson, Rob Fergus, and Ofir Nachum.
\newblock Offline reinforcement learning with fisher divergence critic
  regularization.
\newblock \emph{arXiv preprint arXiv:2103.08050}, 2021{\natexlab{b}}.

\bibitem[Kumar et~al.(2020)Kumar, Zhou, Tucker, and
  Levine]{kumar2020conservative}
Aviral Kumar, Aurick Zhou, George Tucker, and Sergey Levine.
\newblock Conservative q-learning for offline reinforcement learning.
\newblock \emph{arXiv preprint arXiv:2006.04779}, 2020.

\bibitem[Laroche et~al.(2019)Laroche, Trichelair, and
  Des~Combes]{laroche2019safe}
Romain Laroche, Paul Trichelair, and Remi~Tachet Des~Combes.
\newblock Safe policy improvement with baseline bootstrapping.
\newblock In \emph{International Conference on Machine Learning}, pages
  3652--3661. PMLR, 2019.

\bibitem[Levine et~al.(2020)Levine, Kumar, Tucker, and Fu]{levine2020offline}
Sergey Levine, Aviral Kumar, George Tucker, and Justin Fu.
\newblock Offline reinforcement learning: Tutorial, review, and perspectives on
  open problems.
\newblock \emph{arXiv preprint arXiv:2005.01643}, 2020.

\bibitem[Liu et~al.(2020)Liu, Swaminathan, Agarwal, and
  Brunskill]{liu2020provably}
Yao Liu, Adith Swaminathan, Alekh Agarwal, and Emma Brunskill.
\newblock Provably good batch reinforcement learning without great exploration.
\newblock \emph{arXiv preprint arXiv:2007.08202}, 2020.

\bibitem[Mnih et~al.(2015)Mnih, Kavukcuoglu, Silver, Rusu, Veness, Bellemare,
  Graves, Riedmiller, Fidjeland, Ostrovski, et~al.]{mnih2015human}
Volodymyr Mnih, Koray Kavukcuoglu, David Silver, Andrei~A Rusu, Joel Veness,
  Marc~G Bellemare, Alex Graves, Martin Riedmiller, Andreas~K Fidjeland, Georg
  Ostrovski, et~al.
\newblock Human-level control through deep reinforcement learning.
\newblock \emph{Nature}, 518\penalty0 (7540):\penalty0 529, 2015.

\bibitem[Nair et~al.(2020)Nair, Dalal, Gupta, and Levine]{nair2020accelerating}
Ashvin Nair, Murtaza Dalal, Abhishek Gupta, and Sergey Levine.
\newblock Accelerating online reinforcement learning with offline datasets.
\newblock \emph{arXiv preprint arXiv:2006.09359}, 2020.

\bibitem[Paszke et~al.(2019)Paszke, Gross, Massa, Lerer, Bradbury, Chanan,
  Killeen, Lin, Gimelshein, Antiga, et~al.]{paszke2019pytorch}
Adam Paszke, Sam Gross, Francisco Massa, Adam Lerer, James Bradbury, Gregory
  Chanan, Trevor Killeen, Zeming Lin, Natalia Gimelshein, Luca Antiga, et~al.
\newblock Pytorch: An imperative style, high-performance deep learning library.
\newblock In \emph{Advances in neural information processing systems}, pages
  8026--8037, 2019.

\bibitem[Peng et~al.(2019)Peng, Kumar, Zhang, and Levine]{peng2019advantage}
Xue~Bin Peng, Aviral Kumar, Grace Zhang, and Sergey Levine.
\newblock Advantage-weighted regression: Simple and scalable off-policy
  reinforcement learning.
\newblock \emph{arXiv preprint arXiv:1910.00177}, 2019.

\bibitem[Siegel et~al.(2020)Siegel, Springenberg, Berkenkamp, Abdolmaleki,
  Neunert, Lampe, Hafner, Heess, and Riedmiller]{Siegel2020Keep}
Noah Siegel, Jost~Tobias Springenberg, Felix Berkenkamp, Abbas Abdolmaleki,
  Michael Neunert, Thomas Lampe, Roland Hafner, Nicolas Heess, and Martin
  Riedmiller.
\newblock Keep doing what worked: Behavior modelling priors for offline
  reinforcement learning.
\newblock In \emph{International Conference on Learning Representations}, 2020.

\bibitem[Sutton and Barto(2018)]{sutton2018reinforcement}
Richard~S Sutton and Andrew~G Barto.
\newblock \emph{Reinforcement learning: An introduction}.
\newblock MIT press, 2018.

\bibitem[Wang et~al.(2018)Wang, Xiong, Han, Liu, Zhang,
  et~al.]{wang2018exponentially}
Qing Wang, Jiechao Xiong, Lei Han, Han Liu, Tong Zhang, et~al.
\newblock Exponentially weighted imitation learning for batched historical
  data.
\newblock In \emph{Advances in Neural Information Processing Systems}, pages
  6288--6297, 2018.

\bibitem[Wang et~al.(2020)Wang, Novikov, Zolna, Merel, Springenberg, Reed,
  Shahriari, Siegel, Gulcehre, Heess, et~al.]{wang2020critic}
Ziyu Wang, Alexander Novikov, Konrad Zolna, Josh~S Merel, Jost~Tobias
  Springenberg, Scott~E Reed, Bobak Shahriari, Noah Siegel, Caglar Gulcehre,
  Nicolas Heess, et~al.
\newblock Critic regularized regression.
\newblock \emph{Advances in Neural Information Processing Systems}, 33, 2020.

\bibitem[Wu et~al.(2019)Wu, Tucker, and Nachum]{wu2019behavior}
Yifan Wu, George Tucker, and Ofir Nachum.
\newblock Behavior regularized offline reinforcement learning, 2019.

\end{thebibliography}

\newpage

\appendix

\section{Related work continued}\label{sec:related_cont}

There are many other policy improvement operators in the offline RL literature. BCQ \citep{fujimoto2018off} and MBS \citep{liu2020provably} maximize the Q value subject to behavior constraints.
These algorithms try to maximize Q subject to some minimum probability under the estimated behavior. In contrast \qfil\ imitates instead of optimizing Q, making it less likely to choose unseen actions while still providing the quantile hyperparameter to allow for more aggressive updates.
SPIBB \citep{laroche2019safe} instead constrains the policy based on uncertainty estimates.
BRAC \citep{wu2019behavior} and TD3+BC \citep{fujimoto2021minimalist} attempt to optimize the estimated Q values subject to some regularization toward the behavior. These approaches are intuitive and effective on some tasks, but by attempting to optimize estimated Q values and without stronger constraints they are potentially more susceptible to selecting out-of-distribution actions.
Other work modifies the policy evaluation step to learn conservative \citep{kumar2020conservative} or regularized \citep{kostrikov2021offline} Q values. Here we focus on modifications to the improvement step. An interesting direction for future work would be to combine modifications to the improvement and evaluation steps.

\section{Experimental details}\label{sec:details}

\subsection{Synthetic experiment}

As explained in the main text, data is sampled from the following distributions:
\begin{align*}
    s \sim U([0,1]), \qquad a|s \sim \frac{s + \epsilon}{2}, \epsilon \sim U([0,1]), \qquad r(s,a) = \begin{cases}1 - |a - (1-s)|& a \in [\frac{s}{2}, \frac{s+1}{2}]\\ -1 & otherwise  \end{cases}.
\end{align*}

Then the algorithm proceeds in three steps: (1) estimate $ \hat \beta$, (2) estimate $ \widehat Q$, and (3) learn $ \hat \pi$. Note that because we are in a contextual bandit problem, $ \widehat Q$ does not depend on the behavior policy $ \beta$. But of course we will only be able to learn $ \widehat Q$ at actions selected by the behavior. These functions are learned to minimize the following objectives, summed over the dataset:
\begin{align}
    \ell_\beta(s,a,r) &= -\log \hat \beta(a|s)\\
    \ell_Q(s,a,r) &= (r - \widehat Q(s,a))^2\\
    \ell_\pi(s,a,r) &= - \1[\widehat Q(s,a) \geq \widehat V_{\tau, \beta}(s)] \log \hat \pi(a|s)
\end{align}
where we estimate $ \widehat V_{\tau, \beta}$ from $ \widehat Q$ and $ \hat \beta$ at a given state $ s$ as follows. First we take 100 samples $ a_1, \dots, a_{100}$ from $ \hat \beta(\cdot|s)$. Then we compute $ \widehat V_{\tau, \beta}(s)$ as the empirical $ \tau$ quantile of the set of real numbers $ \{\widehat Q(s,a_i): 1 \leq i \leq 100\}$.

All networks ($\hat \beta, \widehat Q, \hat \pi$) are MLPs with width 50, depth 2, and ReLU activation functions and are implemented in PyTorch \cite{paszke2019pytorch}. The policy networks output truncated normal distributions that are truncated to the bounds of the action space and log standard deviations are bounded in [-5, 0]. We train using the Adam optimizer \citep{kingma2014adam} with learning rate 0.001 and batch size 64 for all networks. Each network is trained for 1000 gradient steps. Evaluation for each dataset is conducted by sampling 100 states, then sampling one action $ a\sim \hat \pi(\cdot|s)$, and then evaluating the reward function for each $ s,a$ pair. This gives an estimated reward for one random seed. We report mean and standard deviation over 50 random seeds in Figure \ref{fig:toy}.

\subsection{D4RL experiments}

Datasets are all derived from the D4RL benchmark suite \citep{fu2020d4rl}. We consider two algorithmic outlines: one-step and iterative. We generally follow the hyperparameter choices from \cite{brandfonbrener2021offline}. 

\paragraph{One-step.} For the one-step experiments the algorithm proceeds in a similar three step manner as the synthetic experiment: (1) estimate $ \hat \beta$, (2) estimate $ \widehat Q^\beta$, and (3) learn $ \hat \pi$. We use SARSA Q estimation \citep{sutton2018reinforcement} with target networks \citep{mnih2015human} for increased stability. We initialize $ \hat \pi$ at $ \hat \beta$. The loss functions at each sample $ (s,a,r, s', a')$ from the replay buffer are now:
\begin{align}
    \ell_\beta(s,a,r,s',a') &= -\log \hat \beta(a|s)\\
    \ell_{Q^\beta}(s,a,r,s',a') &= (\widehat Q^\beta(s,a) - r - \gamma  \underline{\widehat Q^\beta}(s', a'))^2\\
    \ell_\pi(s,a,r, s', a') &= - \1[\widehat Q^\beta(s,a) \geq \widehat V^\beta_{\tau, \beta}(s)] \log \hat \pi(a|s)
\end{align}
where $ \underline{\widehat Q^\beta}$ denotes the target network. 

All networks ($\hat \beta, \widehat Q^\beta, \hat \pi$) are MLPs with width 1024, depth 2, and ReLU activation functions. The policy networks output truncated normal distributions that are truncated to the bounds of the action space and log standard deviations are bounded in [-5, 0]. We train using the Adam optimizer with learning rate 0.0001 and batch size 512 for all networks. The target network weights are updated every 2 gradient steps using an exponentially weighted moving average with hyperaparameter 0.005.
The discount factor is $ \gamma = 0.99$.

We train $ \hat \beta$ for 500k gradient steps, $ \widehat Q^\beta$ for 2 million gradient steps, and $ \hat \pi$ for 100k gradient steps. 

For \qfil\ we sweep the quantile $ \tau$ over $[0.5, 0.75, 0.9, 0.95]$ and use 100 samples to estimate $ \widehat V_{\tau, \beta}^\beta(s)$. For the exponentially weighted baseline we sweep the temperature parameter $ \alpha$ over $ [0.3, 1.0, 3.0, 10.0]$, use 10 samples to estimate $ \widehat V_\beta(s)$, and clip the weights at 100 for numerical stability.

We train the entire procedure with 3 random seeds for each of the 6 hyperparameter values. Then we evaluate 100 episodes using the modal actions from $ \hat \pi$ and calculate the mean return. We report the mean and standard deviation over seeds for the best hyperparameter value in the tables in the paper. 

\paragraph{Iterative.} For the iterative antmaze experiments we generally leave all of the hyperparameters the same, but make the following modifications to the learning procedure. Instead of learning $ \hat \pi$ as step 3, we now initialize $ \widehat Q^\pi$ at $ \widehat Q^\beta$ and then iteratively update $ \hat \pi$ and $ \widehat Q^\pi$. In each iteration we take 1 gradient step to update $ \pi$ and 2 gradient steps to update $ \widehat Q^\pi$. The loss functions are now:
\begin{align}
    \ell_{Q^\pi}(s,a,r,s',a') &= (\widehat Q^\pi(s,a) - r - \gamma  \underline{\widehat Q^\pi}(s', \underline{a'}))^2, \quad \underline{a'} \sim \hat \pi(\cdot|s')\\
    \ell_\pi(s,a,r,s', a') &= - \1[\widehat Q^\pi(s,a) \geq \widehat V^\pi_{\tau, \beta}(s)] \log \hat \pi(a|s).
\end{align}
We run this for 300k steps which corresponds to 300k gradient steps on $ \hat \pi$ and 600k gradient steps on $ \widehat Q^\pi$. We again sweep hyperparameters sets of $[0.75, 0.9, 0.95, 0.99]$ and $ [0.3, 1.0, 3.0, 10.0]$ for \qfil\ and the exponentially weighted baseline respectively.

\section{Proof of Proposition 1}\label{sec:proof}

The proof first needs two lemmas, one novel and one from \cite{achiam2017constrained}.

\begin{lemma}[Advantage]\label{lem:advantage}
Under the above assumptions and letting $ W_1$ denote the Wasserstein-1 distance, \qfil\ returns a policy $ \hat \pi$ such that whp
\begin{align}
   \E_{\substack{s \sim d_\beta\\ a\sim \hat \pi|s}}[A^\beta(s,a)] &\geq \E_{s\sim d_\beta} \left[ W_1\left(\widehat Q^\beta_{\sharp, \pi_\tau}(\cdot|s), \ \  \widehat Q^\beta_{\sharp, \beta}(\cdot|s)\right) \right]\\&\qquad - 2\sqrt{\frac{\varepsilon_Q(N)}{1-\tau}} \ - \  \frac{2 Q^\beta_{\max} \varepsilon_\pi((1-\tau) N) }{1 - \tau} . 
\end{align}
\end{lemma}

\begin{proof}
    At a high level, the proof decomposes into three parts, one for each term. The first term comes from the expected improvement of the target policy $ \pi_\tau$ under the estimated Q function $ \widehat Q^\beta$. The second term comes from error in value estimation and the last term from error in imitation learning. First we will set up the decomposition, then we will go through each step.
    
    For the rest of the proof, unless noted otherwise, we will assume that $ s \sim d_\beta$.
    
    \textbf{Decomposition.} We begin by decomposing the expected advantage:
    \begin{align}
        \E_{\substack{s \\ a\sim \hat \pi|s}}[A^\beta(s,a)] &= \E_{\substack{s \\ a\sim \hat \pi|s}}[Q^\beta(s,a)] - \E_{\substack{s \\ a\sim \beta |s}}[Q^\beta(s,a)]\\
        &= \E_{\substack{s \\ a\sim  \pi_\tau |s}}[Q^\beta(s,a)] - \E_{\substack{s \\ a\sim \beta |s}}[Q^\beta(s,a)] + \underbrace{\E_{\substack{s \\ a\sim \hat \pi|s}}[Q^\beta(s,a)] - \E_{\substack{s \\ a\sim  \pi_\tau |s}}[Q^\beta(s,a)]}_{T_\pi}\\
        &= \underbrace{\E_{\substack{s \\ a\sim \pi_\tau|s}}[\widehat Q^\beta(s,a)] - \E_{\substack{s\\ a\sim \beta |s}}[\widehat Q^\beta(s,a)]}_{T_W} \\ & \qquad + \underbrace{\E_{\substack{s\\ a\sim \pi_\tau|s}}[Q^\beta(s,a) -  \widehat Q^\beta(s,a) ] - \E_{\substack{s\\ a\sim \beta |s}}[Q^\beta(s,a) - \widehat Q^\beta(s,a)]}_{T_Q} + T_\pi\\
        &= T_W + T_Q + T_\pi
    \end{align}

    \textbf{Wasserstein term.} First, we show that by the definition of the pushforward distribution, we have that 
    \begin{align}
        T_W = \E_{\substack{s\\ a\sim \pi_\tau|s}}[\widehat Q^\beta(s,a)] 
        - \E_{\substack{s \\ a\sim \beta |s}}[\widehat Q^\beta(s,a)] 
        &= \E_{\substack{s \\ v\sim \widehat Q^\beta_{\sharp, \pi_\tau}|s}} [v] 
        - \E_{\substack{s \\ v\sim \widehat Q^\beta_{\sharp, \beta}|s}} [v] \\
        &=  \E_{s} \left[ \E_{\substack{v\sim \widehat Q^\beta_{\sharp, \pi_\tau}|s } } [v] 
        - \E_{\substack{v\sim \widehat Q^\beta_{\sharp, \beta} | s } } [v]\right]
    \end{align}
    Now, we need to transform this into the Wasserstein-1 distance. Since the values are 1-dimensional real values, we know that $ W_1(P, Q)$ can be written as $ \int_{0}^1 |F_P^{-1}(z) - F_Q^{-1}(z)| dz$ where $ F_P^{-1}$ is the inverse CDF of $ P$. We claim that the inverse CDF of $ \widehat Q^\beta_{\sharp, \pi_\tau}(\cdot|s) $ is strictly greater than the inverse CDF of $ \widehat Q^\beta_{\sharp, \beta}(\cdot|s)$, allowing us to rewrite the above expression as a Wasserstein-1 distance. 
    
    Note that this is equivalent to showing that the CDF of $ \widehat Q^\beta_{\sharp, \pi_\tau}(\cdot|s) $ is strictly less than the CDF of $ \widehat Q^\beta_{\sharp, \beta}(\cdot|s)$. To show this, let $ F_{\pi_\tau}$ denote the CDF of $ \widehat Q^\beta_{\sharp, \pi_\tau}(\cdot|s)$ and $ F_\beta$ denote the CDF of the CDF of $ \widehat Q^\beta_{\sharp, \beta}(\cdot|s)$. By the definition of $ \pi_\tau $ we have that 
    \begin{align}
        F_{\pi_\tau}(v) = \begin{cases}0 & v < \widehat V^\beta_{\tau,\beta}(s)\\ \frac{1}{1-\tau} (F_\beta(v) - \tau) & otherwise \end{cases}
    \end{align}
    We need to show that for all $ v $ we have $ F_{\pi_\tau}(v) \leq F_{\beta}(v)$. Break this into two cases. (1) If $ v < V^\beta_{\tau,\beta}(s) $ then the inequality clearly holds since $ F_{\pi_\tau}(v) = 0$ and $ F_{\beta}(v) \geq 0$. (2) If $ v \geq V^\beta_{\tau,\beta}(s)$, we note that $ F_{\pi_\tau}(v) $ is an affine transformation of $ F_\beta(v)$ with a positive coefficient $ \frac{1}{1-\tau} \geq 1$. 
    Note that at $ v = V^\beta_{\tau,\beta}(s) $ we have $ F_\beta(v) = \tau$ and $ F_{\pi_\tau} = 0$ and that the lowest $ v$ such that $ F_\beta(v) = 1$ also is the lowest $ v $ such that $ F_{\hat \pi}(v) = 1$. Since the CDFs are just an affine transformation of each other and we have that $ F_\beta(v) \geq F_{\pi_\tau}(v)$ at the two end points of the segment, we know that $ F_{\pi_\tau}(v) \leq F_{\beta}(v)$ for all $ v $ as desired.
    
    Thus, we have that 
    \begin{align}
        T_W &=  \E_{s} \left[ \E_{\substack{v\sim \widehat Q^\beta_{\sharp, \pi_\tau}|s } } [v] 
        - \E_{\substack{v\sim \widehat Q^\beta_{\sharp, \beta} | s } } [v]\right] \\
        &= \E_{s} \left[ \int_0^1 F_{\pi_\tau}^{-1}(z) - F_{\beta}^{-1}(z) dz\right] = \E_{s} \left[ \int_0^1 |F_{ \pi_\tau }^{-1}(z) - F_{\beta}^{-1}(z)| dz\right] \\
        &=  \E_{s} \left[ W_1\left(\widehat Q^\beta_{\sharp, \pi_\tau}(\cdot|s), \ \ \widehat Q^\beta_{\sharp, \beta}(\cdot|s)\right)\right]
    \end{align}

    \textbf{Imitation learning term.} It will be useful to define some random variables. Let $ S $ the random variable of the state distributed according to $ \mathcal{D}$. Let $ A$ be the action according the the behavior policy $ \beta(\cdot | S)$. Let $ E$ be the event $ \widehat Q^\beta (S,A) \geq V^\beta_{\tau,\beta}(S)$.  

    We can define the conditional distribution that generates this data according to Bayes' rule:
    \begin{align}
        \Prob(A=a|S=s, E) = \frac{\Prob(A=a|S=s)\Prob(E|A=a, S=s) }{\Prob(E|S=s)} = \beta(a|s) \frac{\1[\widehat Q^\beta(s,a) \geq V^\beta_{\tau,\beta}(s)]}{\Prob(E|S=s)}
    \end{align}
    After applying the filter, we say data generated from $ s|E$ by the policy $ \pi_\tau$ defined as 
    \begin{align}
        \pi_\tau (a|s) = \beta(a|s) \frac{\1[\widehat Q^\beta(s,a) \geq V^\beta_{\tau,\beta}(s)]}{\Prob(E|S=s)} = \beta(a|s) \frac{\1[\widehat Q^\beta(s,a) \geq V^\beta_{\tau,\beta}(s)]}{1-\tau} 
    \end{align}
    By definition of the filtering process we know that the filtering selects a dataset of expected size $ (1-\tau) N$ for the imitation learning step.
    By assumption our algorithm outputs $ \hat \pi$ such that
    \begin{align}
        \E_{s|E}[D_{TV}(\hat \pi(\cdot|s)\| \pi_\tau(\cdot|s))] \leq \varepsilon_\pi((1-\tau) N).
    \end{align}
    Note that states are generated conditioned on $ E$.
    Then, applying importance weighting, Holder's inequality, and the fact that $ p(E|s) = 1-\tau$ we get
    \begin{align}\label{eq:pi_bound}
        - T_\pi &= \E_{\substack{s\\a\sim \pi_\tau|s}}[Q^\beta(s,a)] - \E_{\substack{s\\a\sim \hat \pi|s}}[Q^\beta(s,a)] \\
        &= \E_s \int_a \bigg( \pi_\tau(a|s) Q^\beta(s,a) - \hat \pi(a|s) Q^\beta(s,a)\bigg)\\
        &\leq Q^\beta_{\max} \int_s p(s) \int_a | \pi_\tau(a|s) - \hat \pi(a|s)|\\
        &= Q^\beta{\max} \int_s \frac{p(E)}{p(E|s)} p(s|E) \int_a | \pi_\tau(a|s) - \hat \pi(a|s)|\\
        &= Q^\beta_{\max} \E_{s|E}\frac{p(E)}{p(E|s)}  \int_a | \pi_\tau(a|s) - \hat \pi(a|s)|\\
        &\leq Q^\beta_{\max} \sup_{s}\frac{1}{p(E|s)} \E_{s|E} \int_a |\pi_\tau(a|s) - \hat \pi(a|s)|\\
        &= \frac{2 Q^\beta_{\max}}{1-\tau} \E_{s|E}[D_{TV}(\hat \pi(\cdot|s)\| \pi_\tau(\cdot|s))]\\
        &\leq \frac{2 Q^\beta_{\max}}{1-\tau} \varepsilon_\pi((1-\tau)N) .  \label{eq:pi_bound}
    \end{align}

    \textbf{Value estimation term.} We bound $ T_Q$ in two parts. When actions are generated from $ \beta$ we can just use Jensen's inequality and apply our value estimation assumption to bound 
    \begin{align}
        \left| \E_{\substack{s\\a\sim \beta|s}}\left[Q^\beta(s,a) - \widehat Q^\beta(s,a)\right] \right|\leq \sqrt{ \E_{\substack{s\\a\sim \beta|s}}\left[(Q^\beta(s,a) - \widehat Q^\beta(s,a))^2\right]} \leq \sqrt{\varepsilon_Q(N)}
    \end{align}
    
    When actions are generated from $ \pi_\tau$ we need to be more careful about distribution shift. Explicitly we use importance weighting, Cauchy-Schwarz, the fact that $ p(E|s) = 1-\tau$, and our value estimation assumption  to get that
    \begin{align}
        \bigg|\E_{\substack{s\\a\sim \pi_\tau|s}}&[ Q^\beta(s,a) - \widehat Q^\beta(s,a)]\bigg| = \left|\E_{\substack{s\\a\sim \beta|s}}\left[\frac{\pi_\tau(a|s)}{\beta(a|s)} ( Q^\beta(s,a) - \widehat Q^\beta(s,a))\right]\right|\\
        &\leq \sqrt{\E_{\substack{s\\a\sim \beta|s}}\left[\frac{\pi_\tau(a|s)^2}{\beta(a|s)^2}\right] \E_{\substack{s\\a\sim \beta|s}}\left[(Q^\beta(s,a) - \widehat Q^\beta(s,a))^2\right]}
        \\&= \sqrt{\E_{\substack{s\\a\sim \beta|s}}\left[ \frac{\1[\widehat Q^\beta(s,a) \geq V^\beta_{\tau,\beta}(s)]^2}{p(E|s)^2}\right] \E_{\substack{s\\a\sim \beta|s}}\left[(Q^\beta(s,a) - \widehat Q^\beta(s,a))^2\right]}
        \\&= \sqrt{\E_{\substack{s}} \frac{1}{{(1-\tau)^2}}\E_{a\sim \beta|s}\left[\1[\widehat Q^\beta(s,a) \geq V^\beta_{\tau,\beta}(s)]\right] \E_{\substack{s\\a\sim \beta|s}}\left[( Q^\beta(s,a) - \widehat Q^\beta(s,a))^2\right]}
        \\&= \sqrt{\E_{\substack{s}} \frac{1-\tau }{{(1-\tau)^2}} \E_{\substack{s\\a\sim \beta|s}}\left[( Q^\beta(s,a) - \widehat Q^\beta(s,a))^2\right]}\\
        &\leq \sqrt{\frac{\varepsilon_Q(N)}{1-\tau}}\label{eq:q_bound}
    \end{align}
    Since $ 1 \leq \sqrt{\frac{1}{1-\tau}}$, we can combine the two parts to bound $ -T_Q \leq 2\sqrt{\frac{ \varepsilon_Q(N)}{1-\tau}}$.

Combining the bounds on $ T_W, T_\pi$, and $ T_Q$ gives us the desired result
\end{proof}

\begin{lemma}[\cite{achiam2017constrained}]\label{lem:achiam}
Under the same assumptions as above, further define $ A^{\beta, \pi}_\infty = \sup_s |\E_{a\sim  \pi}[A^\beta(s,a)]|$
\begin{align}
    J(\pi) - J(\beta) \geq \frac{1}{1-\gamma} \left(\mathbb{A}^\beta( \pi) - \frac{2\gamma A^{\beta, \pi}_\infty }{1 - \gamma} \E_{s\sim d^\beta}\left[ D_{TV}\big( \pi(\cdot|s) \| \beta(\cdot|s)\big)\right]\right).
\end{align}
\end{lemma}

\advantage*

\begin{proof}
    The proposition follows directly from plugging Lemma \ref{lem:advantage} into the advantage term of Lemma \ref{lem:achiam} and then bounding the TV term, which we will do now. 
    
    \begin{align}
        \E_s \left[ D_{TV}\big( \hat \pi(\cdot|s) \| \beta(\cdot|s)\big)\right] &\leq \E_s\left[ D_{TV}\big( \hat \pi(\cdot|s) \| \pi_\tau(\cdot|s)\big)\right] + \E_s\left[ D_{TV}\big( \pi_\tau(\cdot|s) \| \beta(\cdot|s)\big)\right]\\
        &\leq \varepsilon((1-\tau)N) + \frac{1}{2} \E_s\left[\int_a |\pi_\tau(a|s) - \beta(a|s)| \right]\\
        &= \varepsilon((1-\tau)N) + \frac{1}{2} \E_s\bigg[\int_a |\pi_\tau(a|s) - \beta(a|s)|\1[\widehat Q^\beta(s,a) < V^\beta_{\tau,\beta}(s)] \\ &\qquad \qquad\qquad + \int_a |\pi_\tau(a|s) - \beta(a|s)|\1[\widehat Q^\beta(s,a) \geq V^\beta_{\tau,\beta}(s)] \bigg]\\
        &= \varepsilon((1-\tau)N) + \frac{1}{2} \E_s\bigg[\int_a \beta(a|s) \1[\widehat Q^\beta(s,a) < V^\beta_{\tau,\beta}(s)] \\ &\qquad \qquad\qquad + \int_a (\frac{\beta(a|s)}{1-\tau} - \beta(a|s))\1[\widehat Q^\beta(s,a) \geq V^\beta_{\tau,\beta}(s)] \bigg]\\
        &= \varepsilon((1-\tau)N) + \frac{1}{2} \E_s\bigg[\tau + \frac{\tau}{1-\tau}\int_a \beta(a|s)\1[\widehat Q^\beta(s,a) \geq V^\beta_{\tau,\beta}(s)] \bigg]\\
         &= \varepsilon((1-\tau)N) + \frac{1}{2} \E_s\bigg[\tau + \frac{\tau}{1-\tau}(1-\tau) \bigg] \\
         &= \varepsilon((1-\tau)N) + \tau
    \end{align}
   Plugging this in yields the desired bound. 
\end{proof}

\end{document}